\documentclass[11pt,a4paper]{article}
\usepackage[hyperref]{naaclhlt2019}
\usepackage{times}
\usepackage{latexsym}
\usepackage{url}            %
\usepackage{nicefrac}       %
\usepackage{microtype}      %

\usepackage{amsfonts}
\usepackage{amsmath}
\usepackage{amssymb}
\usepackage{mathtools}

\usepackage{graphicx}
\usepackage[font=small]{caption}
\usepackage{subcaption}
\usepackage{adjustbox}
\usepackage{wrapfig}
\graphicspath{{figures/pdf/}}
\DeclareGraphicsExtensions{.pdf}

\usepackage{tikz}
\usetikzlibrary{arrows}
\usetikzlibrary{backgrounds}

\usepackage{algorithm}
\usepackage{algorithmic}
\usepackage{rotating}

\usepackage{booktabs}
\usepackage{multirow}
\usepackage{tabularx}
\usepackage{makecell}

\usepackage{balance}
\usepackage{enumitem}
\usepackage{hyperref}
\usepackage{setspace}

\usepackage{macro}

\usepackage{color}
\definecolor{mdgreen}{rgb}{0.05,0.6,0.05}

\aclfinalcopy %

\title{Consistency by Agreement in Zero-shot Neural Machine Translation}

\author{%
  Maruan Al-Shedivat\thanks{~~Work done at Google.} \\
  Carnegie Mellon University \\
  Pittsburgh, PA 15213 \\
  {\tt alshedivat@cs.cmu.edu} \\
  \And
  Ankur P. Parikh \\
  Google AI Language \\
  New York, NY 10011 \\
  {\tt aparikh@google.com} \\}

\date{}

\begin{document}
\maketitle
\begin{abstract}
Generalization and reliability of multilingual translation often highly depend on the amount of available parallel data for each language pair of interest.
In this paper, we focus on zero-shot generalization---a challenging setup that tests models on translation directions they have not been optimized for at training time.
To solve the problem, we
(i) reformulate multilingual translation as probabilistic inference,
(ii) define the notion of zero-shot consistency and show why standard training often results in models unsuitable for zero-shot tasks, and
(iii) introduce a consistent agreement-based training method that encourages the model to produce equivalent translations of parallel sentences in auxiliary languages.
We test our multilingual NMT models on multiple public zero-shot translation benchmarks (IWSLT17, UN corpus, Europarl) and show that agreement-based learning often results in 2-3 BLEU zero-shot improvement over strong baselines without any loss in performance on supervised translation directions.
\end{abstract}

\section{Introduction}

Machine translation (MT) has made remarkable advances with the advent of deep learning approaches \citep{bojar2016findings,wu2016google,crego2016systran,junczys2016neural}.
The progress was largely driven by the encoder-decoder framework \citep{sutskever2014sequence,cho2014learning} and typically supplemented with an attention mechanism \citep{bahdanau2014neural,luong2015effective}.

Compared to the traditional phrase-based systems \citep{koehn2009statistical}, neural machine translation (NMT) requires large amounts of data in order to reach high performance \citep{koehn2017six}.
Using NMT in a multilingual setting exacerbates the problem by the fact that given $k$ languages translating between all pairs would require $O(k^2)$ parallel training corpora (and $O(k^2)$ models).

In an effort to address the problem, different multilingual NMT approaches have been proposed recently.
\citet{luong2015multi, firat2016multi} proposed to use $O(k)$ encoders/decoders that are then intermixed to translate between language pairs.
\citet{johnson2016google} proposed to use a single model and prepend special symbols to the source text to indicate the target language, which has later been extended to other text preprocessing approaches \citep{ha2017effective} as well as language-conditional parameter generation for encoders and decoders of a single model \citep{platanios2018contextual}.

\begin{figure}[t]
    \centering
    \includegraphics[width=0.95\columnwidth]{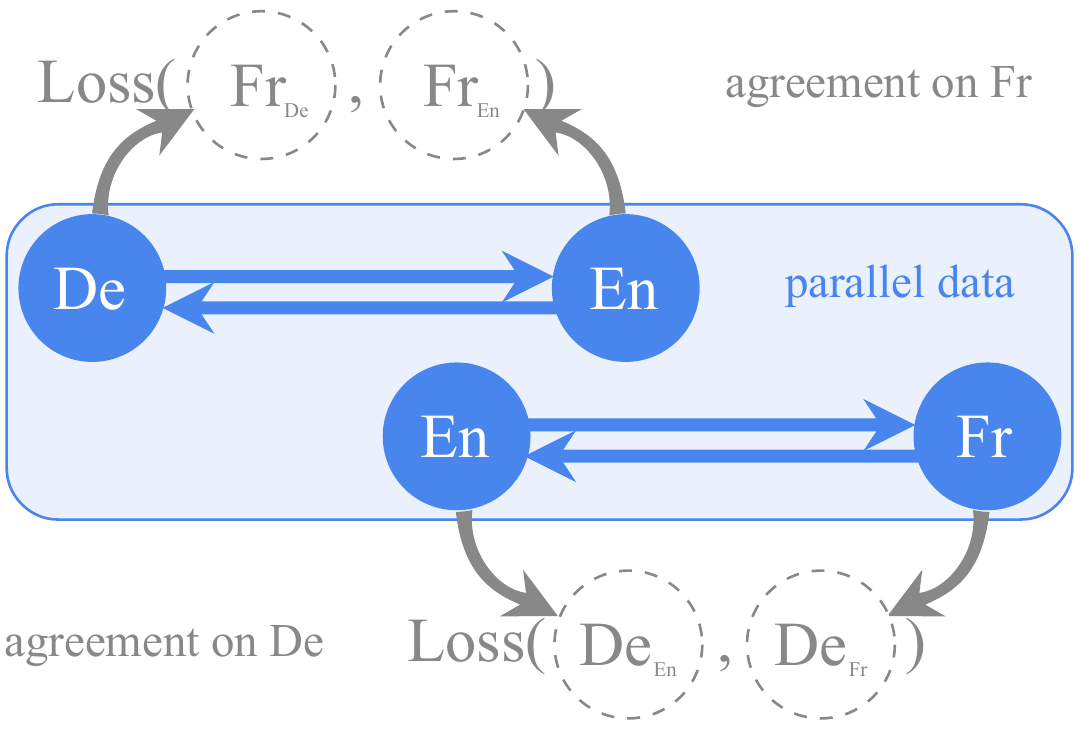}
    \vspace{-1ex}
    \caption{Agreement-based training of a multilingual NMT.
    At training time, given English-French ($\En \leftrightarrow \Fr$) and English-German ($\En \leftrightarrow \De$) parallel sentences, the model not only is trained to translate between the pair but also to agree on translations into a third language.}
    \label{fig:agreement-illustration}
    \vspace{-2.5ex}
\end{figure}

\citet{johnson2016google} also show that a single multilingual system could potentially enable \textit{zero-shot} translation, \ie, it can translate between language pairs not seen in training.
For example, given 3 languages---German (\De), English (\En), and French (\Fr)---and training parallel data only for (\De, \En) and (\En, \Fr), at test time, the system could additionally translate between (\De, \Fr).

Zero-shot translation is an important problem.
Solving the problem could significantly improve data efficiency---a single multilingual model would be able to generalize and translate between any of the $O(k^2)$ language pairs after being trained only on $O(k)$ parallel corpora.
However, performance on zero-shot tasks is often unstable and significantly lags behind the supervised directions.
Moreover, attempts to improve zero-shot performance by fine-tuning~\cite{firat2016zero,sestorain2018zero} may negatively impact other directions.

In this work, we take a different approach and aim to improve the training procedure of \citet{johnson2016google}.
First, we analyze multilingual translation problem from a probabilistic perspective and define the notion of \textit{zero-shot consistency} that gives insights as to why the vanilla training method may not yield models with good zero-shot performance.
Next, we propose a novel training objective and a modified learning algorithm that achieves consistency via agreement-based learning \citep{liang2006alignment,liang2008agreement} and improves zero-shot translation.
Our training procedure encourages the  model to  produce  equivalent  translations  of  parallel training sentences into an auxiliary language (Figure~\ref{fig:agreement-illustration}) and is provably zero-shot consistent.
In addition, we make a simple change to the neural decoder to make the agreement losses fully differentiable.

We conduct experiments on IWSLT17~\citep{mauro2017overview}, UN corpus~\citep{ziemski2016united}, and Europarl~\citep{koehn2017europarl}, carefully removing complete pivots from the training corpora.
Agreement-based learning results in up to +3 BLEU zero-shot improvement over the baseline, compares favorably (up to +2.4 BLEU) to other approaches in the literature \citep{cheng2017joint, sestorain2018zero}, is competitive with pivoting, and does not lose in performance on supervised directions.

\section{Related work}
\label{sec:related-work}

A simple (and yet effective) baseline for zero-shot translation is pivoting that chain-translates, first to a pivot language, then to a target \citep{cohn2007machine, wu2007pivot, utiyama2007comparison}.
Despite being a pipeline, pivoting gets better as the supervised models improve, which makes it a strong baseline in the zero-shot setting.
\citet{cheng2017joint} proposed a joint pivoting learning strategy that leads to further improvements.

\citet{lu2018neural} and \citet{ari2018zero} proposed different techniques to obtain ``neural interlingual'' representations that are passed to the decoder.
\citet{sestorain2018zero} proposed another fine-tuning technique that uses dual learning \citep{he2016dual}, where a language model is used to provide a signal for fine-tuning zero-shot directions.

Another family of approaches is based on distillation \citep{hinton2014dark, kim2016sequence}.
Along these lines, \citet{firat2016zero} proposed to fine tune a multilingual model to a specified zero-shot-direction with pseudo-parallel data and  \citet{chen2017teacher} proposed a teacher-student framework.
While this can yield solid performance improvements, it also adds multi-staging overhead and often does not preserve performance of a single model on the supervised directions.
We note that our approach (and agreement-based learning in general) is somewhat similar to distillation at training time, which has been explored for large-scale single-task prediction problems \citep{anil2018large}.

A setting harder than zero-shot is that of fully unsupervised translation \citep{ravi2011deciphering, artetxe2017unsupervised, lample2017unsupervised, lample2018phrase} in which no parallel data is available for training.
The ideas proposed in these works (\eg, bilingual dictionaries \citep{conneau2017word}, backtranslation \citep{sennrich2015improving} and language models \citep{he2016dual}) are complementary to our approach, which encourages agreement among different translation directions in the zero-shot multilingual setting.

\section{Background}
\label{sec:background}

We start by establishing more formal notation and briefly reviewing some background on encoder-decoder multilingual machine translation from a probabilistic perspective.

\subsection{Notation}
\label{sec:notation}

\paragraph{Languages.}
We assume that we are given a collection of $k$ languages, $L_1, \dots, L_k$, that share a common vocabulary, $V$.
A language, $L_i$, is defined by the marginal probability $\prob{\xv_i}$ it assigns to sentences (\ie, sequences of tokens from the vocabulary), denoted $\xv_i := (x^1, \dots, x^l)$, where $l$ is the length of the sequence.
All languages together define a joint probability distribution, $\prob{\xv_1, \dots, \xv_k}$, over $k$-tuples of \emph{equivalent sentences}.

\paragraph{Corpora.}
While each sentence may have an equivalent representation in all languages, we assume that we have access to only partial sets of equivalent sentences, which form \emph{corpora}.
In this work, we consider \emph{bilingual} corpora, denoted $\Cc_{ij}$, that contain pairs of sentences sampled from $\prob{\xv_i, \xv_j}$ and \emph{monolingual} corpora, denoted $\Cc_i$, that contain sentences sampled from $\prob{\xv_i}$.

\paragraph{Translation.}
Finally, we define a \emph{translation task} from language $L_i$ to $L_j$ as learning to model the conditional distribution $\prob{\xv_j \mid \xv_i}$.
The set of $k$ languages along with translation tasks can be represented as a directed graph $\Gc (\Vc, \Ec)$ with a set of $k$ nodes, $\Vc$, that represent languages and edges, $\Ec$, that indicate translation directions.
We further distinguish between two disjoint subsets of edges: (i) supervised edges, $\Ec_s$, for which we have parallel data, and (ii) zero-shot edges, $\Ec_0$, that correspond to zero-shot translation tasks.
Figure~\ref{fig:translation-graph} presents an example translation graph with supervised edges ($\En \leftrightarrow \Es$, $\En \leftrightarrow \Fr$, $\En \leftrightarrow \Ru$) and zero-shot edges ($\Es \leftrightarrow \Fr$, $\Es \leftrightarrow \Ru$, $\Fr \leftrightarrow \Ru$).
We will use this graph as our running example.

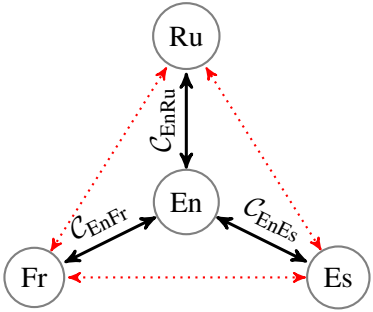
\begin{figure}[t]
    \centering
    \begin{tikzpicture}[->, >=stealth', shorten >=1pt, auto, node distance=1cm, thick, main/.style={circle, draw=black!50, minimum width=15pt}]
        \node[main] (L1) at (0, 0) {En};
        \node[main] (L2) at (2, -1) {Es};
        \node[main] (L3) at (-2, -1) {Fr};
        \node[main] (L4) at (0, 2.2) {Ru};
        
        \draw [<->, line width=0.4mm] (L1) -- (L2) node[midway, above, sloped] {$\Cc_\text{EnEs}$};
        \draw [<->, line width=0.4mm] (L1) -- (L3) node[midway, above, sloped] {$\Cc_\text{EnFr}$};
        \draw [<->, line width=0.4mm] (L1) -- (L4) node[midway, above, sloped] {$\Cc_\text{EnRu}$};
        
        \draw [<->, dotted, color=red] (L2) -- (L4);
        \draw [<->, dotted, color=red] (L2) -- (L3);
        \draw [<->, dotted, color=red] (L3) -- (L4);
    \end{tikzpicture}
    \vspace{-1ex}
    \caption{Translation graph: Languages (nodes), parallel corpora (solid edges), and zero-shot directions (dotted edges).}
    \label{fig:translation-graph}
    \vspace{-1ex}
\end{figure}

\subsection{Encoder-decoder framework}
\label{sec:enc-dec}

First, consider a purely bilingual setting, where we learn to translate from a source language, $L_s$, to a target language, $L_t$.
We can train a translation model by optimizing the conditional log-likelihood of the bilingual data under the model:
\begin{equation}
    \vspace{-1ex}
    \label{eq:bilingual-cond-lik-opt}
    \hat \theta := \argmax\nolimits_\theta \sum\nolimits_{\Cc_{st}} \log \prob[\theta]{\xv_t \mid \xv_s}
    \vspace{-0.5ex}
\end{equation}
where $\hat \theta$ are the estimated parameters of the model.

The encoder-decoder framework introduces a latent sequence, $\uv$, and represents the model as:
\begin{equation}
    \label{eq:enc-dec-cond-lik}
    \prob[\theta]{\xv_t \mid \xv_s} = \probs{\theta}{\mathrm{dec}}{\xv_t \mid \uv = f_{\theta}^{\mathrm{enc}}(\xv_s)}
\end{equation}
where $f_{\theta}^{\mathrm{enc}}(\xv_s)$ is the encoder that maps a source sequence to a sequence of latent representations, $\uv$, and the decoder defines $\probs{\theta}{\mathrm{dec}}{\xv_t \mid \uv}$.\footnote{Slightly abusing the notation, we use $\theta$ to denote all parameters of the model: embeddings, encoder, and decoder.}
Note that $\uv$ is usually deterministic with respect to $\xv_s$ and accurate representation of the conditional distribution highly depends on the decoder.
In neural machine translation, the exact forms of encoder and decoder are specified using RNNs~\citep{sutskever2014sequence}, CNNs~\citep{gehring2016convolutional}, and attention~\citep{bahdanau2014neural,vaswani2017attention} as building blocks.
The decoding distribution, $\probs{\theta}{\mathrm{dec}}{\xv_t \mid \uv}$, is typically modeled autoregressively.

\subsection{Multilingual neural machine translation}
\label{sec:multilingual-nmt}

In the multilingual setting, we would like to learn to translate in \emph{all directions} having access to only few parallel bilingual corpora.
In other words, we would like to learn a collection of models, $\{\prob[\theta]{\xv_j \mid \xv_i}\}_{i, j \in \Ec}$.
We can assume that models are independent and choose to learn them by maximizing the following objective:
\begin{equation}
    \label{eq:basic-multilingual-nmt-objective}
        \Lc^\mathrm{ind}(\theta) = \sum_{i, j \in \Ec_s} \sum_{(\xv_i, \xv_j) \in \Cc_{ij}} \log \prob[\theta]{\xv_j \mid \xv_i}
\end{equation}
In the statistics literature, this estimation approach is called \emph{maximum composite likelihood}~\citep{besag1975statistical,lindsay1988composite} as it composes the objective out of (sometimes weighted) terms that represent conditional sub-likelihoods (in our example, $\prob[\theta]{\xv_j \mid \xv_i}$).
Composite likelihoods are easy to construct and tractable to optimize as they do not require representing the full likelihood, which would involve integrating out variables unobserved in the data (see Appendix~\ref{app:complete-likelihood}).

\citet{johnson2016google} proposed to train a multilingual NMT systems by optimizing a composite likelihood objective \eqref{eq:basic-multilingual-nmt-objective} while representing all conditional distributions, $\prob[\theta]{\xv_j \mid \xv_i}$, with a \emph{shared} encoder and decoder and using language tags, $l_t$, to distinguish between translation directions:
\begin{equation}
    \label{eq:enc-dec-shared}
    \prob{\xv_t \mid \xv_s} = \probs{\theta}{\mathrm{dec}}{\xv_t \mid \uv_{st} = f_{\theta}^\mathrm{enc}(\xv_s, l_t)}
\end{equation}
This approach has numerous advantages including:
(a) simplicity of training and the architecture (by slightly changing the training data, we convert a bilingual NMT into a multilingual one),
(b) sharing parameters of the model between different translation tasks that may lead to better and more robust representations.
\citet{johnson2016google} also show that resulting models seem to exhibit some degree of zero-shot generalization enabled by parameter sharing.
However, since we lack data for zero-shot directions, composite likelihood \eqref{eq:basic-multilingual-nmt-objective} misses the terms that correspond to the zero-shot models, and hence has no statistical guarantees for performance on zero-shot tasks.\footnote{In fact, since the objective \eqref{eq:basic-multilingual-nmt-objective} assumes that the models are independent, plausible zero-shot performance would be more indicative of the limited capacity of the model or artifacts in the data (\eg, presence of multi-parallel sentences) rather than zero-shot generalization.}

\section{Zero-shot generalization \& consistency}
\label{sec:zero-shot-generalization}

Multilingual MT systems can be evaluated in terms of \emph{zero-shot performance}, or quality of translation along the directions they have not been optimized for (\eg, due to lack of data).
We formally define zero-shot generalization via consistency.
\begin{definition}[Expected Zero-shot Consistency]
\label{def:zero-shot-consistency}
Let $\Ec_s$ and $\Ec_0$ be supervised and zero-shot tasks, respectively.
Let $\ell(\cdot)$ be a non-negative loss function and $\Mc$ be a model with maximum expected supervised loss bounded by some $\varepsilon > 0:$
\begin{equation*}
\max_{(i, j) \in \Ec_s} \ep[\xv_i, \xv_j]{\ell(\Mc)} < \varepsilon
\end{equation*}
We call $\Mc$ zero-shot consistent with respect to $\ell(\cdot)$ if for some $\kappa(\varepsilon) > 0$
\begin{equation*}
    \max_{(i, j) \in \Ec_0}\ep[\xv_i, \xv_j]{\ell(\Mc)} < \kappa(\varepsilon),
\end{equation*}
where $\kappa(\varepsilon) \rightarrow 0$ as $\varepsilon \rightarrow 0$.
\end{definition}
In other words, we say that a machine translation system is zero-shot consistent if low error on supervised tasks implies a low error on zero-shot tasks in expectation (\ie, the system generalizes).
We also note that our notion of consistency somewhat resembles error bounds in the domain adaptation literature~\citep{ben2010theory}.

In practice, it is attractive to have MT systems that are guaranteed to exhibit zero-shot generalization since the access to parallel data is always limited and training is computationally expensive.
While the training method of \citet{johnson2016google} does not have guarantees, we show that our proposed approach is provably zero-shot consistent.

\section{Approach}
\label{sec:approach}

We propose a new training objective for multilingual NMT architectures with shared encoders and decoders that avoids the limitations of pure composite likelihoods.
Our method is based on the idea of agreement-based learning initially proposed for learning consistent alignments in phrase-based statistical machine translation (SMT) systems \citep{liang2006alignment,liang2008agreement}.
In terms of the final objective function, the method ends up being reminiscent of distillation \citep{kim2016sequence}, but suitable for joint multilingual training.

\subsection{Agreement-based likelihood}
\label{sec:agreement-based-likelihood}

To introduce agreement-based objective, we use the graph from  Figure~\ref{fig:translation-graph} that defines translation tasks between 4 languages (\En, \Es, \Fr, \Ru).
In particular, consider the composite likelihood objective \eqref{eq:basic-multilingual-nmt-objective} for a pair of $\En-\Fr$ sentences, $(\xv_\En, \xv_\Fr)$:
\begin{align}
    \label{eq:basic-multilingual-nmt-objective-en-fr}
    \MoveEqLeft \Lc^\mathrm{ind}_{\En\Fr}(\theta) \\
    = & \, \log \left[ \prob[\theta]{\xv_\Fr \mid \xv_\En} \prob[\theta]{\xv_\En \mid \xv_\Fr} \right] \nonumber \\
    = & \, \log \left[
        \sum_{\zv_\Es^\prime, \zv_\Ru^\prime} \prob[\theta]{\xv_\Fr, \zv_\Es^\prime, \zv_\Ru^\prime \mid \xv_\En} \times \right. \nonumber \\[-2ex]
    & \qquad\, \left. \sum_{\zv_\Es^{\prime\prime}, \zv_\Ru^{\prime\prime}} \prob[\theta]{\xv_\En, \zv_\Es^{\prime\prime}, \zv_\Ru^{\prime\prime} \mid \xv_\Fr}
    \right] \nonumber
\end{align}
where we introduced latent translations into Spanish ($\Es$) and Russian ($\Ru$) and marginalized them out (by virtually summing over all sequences in the corresponding languages).
Again, note that this objective assumes independence of $\En \rightarrow \Fr$ and $\Fr \rightarrow \En$ models.

Following \citet{liang2008agreement}, we propose to tie together the single prime and the double prime latent variables, $\zv_\Es$ and $\zv_\Ru$, to encourage agreement between $\prob[\theta]{\xv_\En, \zv_\Es, \zv_\Ru \mid \xv_\Fr}$ and $\prob[\theta]{\xv_\Fr, \zv_\Es, \zv_\Ru \mid \xv_\En}$ on the latent translations.
We interchange the sum and the product operations inside the $\log$ in \eqref{eq:basic-multilingual-nmt-objective-en-fr}, denote $\zv := (\zv_\Es, \zv_\Ru)$ to simplify notation, and arrive at the following new objective function:
\begin{align}
    \label{eq:agreement-multilingual-nmt-objective-general}
    \MoveEqLeft \Lc^\mathrm{agree}_{\En\Fr}(\theta) := \\
    & \log \sum_{\zv} \prob[\theta]{\xv_\Fr, \zv \mid \xv_\En} \prob[\theta]{\xv_\En, \zv \mid \xv_\Fr} \nonumber
\end{align}
Next, we factorize each term as:
\begin{equation*}
    \prob{\xv, \zv \mid \yv} = \prob{\xv \mid \zv, \yv} \prob[\theta]{\zv \mid \yv}
\end{equation*}
Assuming $\prob[\theta]{\xv_\Fr \mid \zv, \xv_\En} \approx \prob[\theta]{\xv_\Fr \mid \xv_\En}$,\footnote{This means that it is sufficient to condition on a sentence in one of the languages to determine probability of a translation in any other language.} the objective \eqref{eq:agreement-multilingual-nmt-objective-general} decomposes into two terms:
\begin{align}
    \label{eq:agreement-multilingual-nmt-objective-ij}
    \MoveEqLeft \Lc^\mathrm{agree}_{\En\Fr}(\theta) \\
    \approx 
    & \, \underbrace{\log \prob[\theta]{\xv_\Fr \mid \xv_\En} + \log \prob[\theta]{\xv_\En \mid \xv_\Fr}}_\text{composite likelihood terms} + \nonumber \\
    & \, \underbrace{\log \sum_{\zv} \prob[\theta]{\zv \mid \xv_\En} \prob[\theta]{\zv \mid \xv_\Fr}}_\text{agreement term} \nonumber
\end{align}
We call the expression given in \eqref{eq:agreement-multilingual-nmt-objective-ij} \emph{agreement-based likelihood}.
Intuitively, this objective is the likelihood of observing parallel sentences $(\xv_\En, \xv_\Fr)$ \emph{and} having sub-models $\prob[\theta]{\zv \mid \xv_\En}$ and $\prob[\theta]{\zv \mid \xv_\Fr}$ agree on all translations into $\Es$ and $\Ru$ at the same time.

\paragraph{Lower bound.}
Summation in the agreement term over $\zv$ (\ie, over possible translations into $\Es$ and $\Ru$ in our case) is intractable. 
Switching back from $\zv$ to $(\zv_\Es, \zv_\Ru)$ notation and using Jensen's inequality, we lower bound it with cross-entropy:\footnote{Note that expectations in \eqref{eq:agreement-term-lower-bound} are conditional on $\xv_\En$. Symmetrically, we can have a lower bound with expectations conditional on $\xv_\Fr$. In practice, we symmetrize the objective.}
\begin{align}
    \label{eq:agreement-term-lower-bound}
    \MoveEqLeft \log \sum_{\zv} \prob[\theta]{\zv \mid \xv_\En} \prob[\theta]{\zv \mid \xv_\Fr} \nonumber \\ %
    \geq & \, \ep[\zv_\Es \mid \xv_\En]{\log \prob[\theta]{\zv_\Es \mid \xv_\Fr}} + \\
    & \, \ep[\zv_\Ru \mid \xv_\En]{\log \prob[\theta]{\zv_\Ru \mid \xv_\Fr}}  \nonumber
\end{align}
We can estimate the expectations in the lower bound on the agreement terms by sampling $\zv_\Es \sim \prob[\theta]{\zv_\Es \mid \xv_\En}$ and $\zv_\Ru \sim \prob[\theta]{\zv_\Ru \mid \xv_\En}$.
In practice, instead of sampling we use greedy, continuous decoding (with a fixed maximum sequence length) that also makes $\zv_\Es$ and $\zv_\Ru$ differentiable with respect to parameters of the model.

\subsection{Consistency by agreement}
\label{sec:agreement-consistency}

We argue that models produced by maximizing agreement-based likelihood \eqref{eq:agreement-multilingual-nmt-objective-ij} are zero-shot consistent.
Informally, consider again our running example from Figure~\ref{fig:translation-graph}.
Given a pair of parallel sentences in $(\En, \Fr)$, agreement loss encourages translations from $\En$ to $\{\Es, \Ru\}$ and translations from $\Fr$ to $\{\Es, \Ru\}$ to coincide.
Note that $\En \rightarrow \{\Es, \Fr, \Ru\}$ are supervised directions.
Therefore, agreement ensures that translations along the zero-shot edges in the graph match supervised translations.
Formally, we state it as:

\begin{theorem}[Agreement Zero-shot Consistency]
\label{thm:agreement-consistency}
Let $L_1$, $L_2$, and $L_3$ be a collection of languages with $L_1 \leftrightarrow L_2$ and $L_2 \leftrightarrow L_3$ be supervised while $L_1 \leftrightarrow L_3$ be a zero-shot direction.
Let $\prob[\theta]{\xv_j \mid \xv_i}$ be sub-models represented by a multilingual MT system.
If the expected agreement-based loss, $\ep[\xv_1, \xv_2, \xv_3]{\Lc^\mathrm{agree}_{12}(\theta) + \Lc^\mathrm{agree}_{23}(\theta)}$, is bounded by some $\varepsilon > 0$, then, under some mild technical assumptions on the true distribution of the equivalent translations, the zero-shot cross-entropy loss is bounded as follows:
\begin{equation*}
\ep[\xv_1, \xv_3]{-\log\prob[\theta]{\xv_3 \mid \xv_1}} \le \kappa (\varepsilon)
\end{equation*}
where $\kappa(\varepsilon) \rightarrow 0$ as $\varepsilon \rightarrow 0$.
\end{theorem}
For discussion of the assumptions and details on the proof of the bound, see Appendix~\ref{app:agreement-consistency-proof}.
Note that Theorem~\ref{thm:agreement-consistency} is straightforward to extend from triplets of languages to arbitrary connected graphs, as given in the following corollary.

\begin{corollary}
Agreement-based learning yields zero shot consistent MT models (with respect to the cross entropy loss) for arbitrary translation graphs as long as supervised directions span the graph.
\end{corollary}

\paragraph{Alternative ways to ensure consistency.}
Note that there are other ways to ensure zero-shot consistency, \eg, by fine-tuning or post-processing a trained multilingual model.
For instance, pivoting through an intermediate language is also zero-shot consistent, but the proof requires stronger assumptions about the quality of the supervised source-pivot model.\footnote{Intuitively, we have to assume that source-pivot model does not assign high probabilities to unlikely translations as the pivot-target model may react to those unpredictably.}
Similarly, using model distillation \citep{kim2016sequence, chen2017teacher} would be also provably consistent under the same assumptions as given in Theorem~\ref{thm:agreement-consistency}, but for a single, pre-selected zero-shot direction.
Note that our proposed agreement-based learning framework is provably consistent for \emph{all} zero-shot directions and does not require any post-processing.
For discussion of the alternative approaches and consistency proof for pivoting, see Appendix~\ref{app:distillation-pivoting-consistency}.

\begin{algorithm}[t]
     \caption{Agreement-based M-NMT training}
     \label{alg:agreement-algorithm}
     \small
     \setstretch{1.07}
     \begin{algorithmic}[1]
         \INPUT Architecture (\texttt{GNMT}), agreement coefficient ($\gamma$)
         \STATE Initialize: $\theta \leftarrow \theta_0$
         \WHILE{not (converged or step limit reached)}
            \STATE Get a mini-batch of parallel src-tgt pairs, $(\Xv_s, \Xv_t)$
            \STATE Supervised loss: $\Lc^\mathrm{sup}(\theta) \leftarrow \log \prob[\theta]{\Xv_t \mid \Xv_s}$
            \STATE Auxiliary languages: $L_a \sim \mathrm{Unif}(\{1, \dots, k\})$
            \STATE Auxiliary translations: \\
            \qquad $\Zv_{a \leftarrow s} \leftarrow \mathrm{Decode}\left(\Zv_a \mid f^\mathrm{enc}_\theta(\Xv_s, L_a)\right)$ \\
            \qquad $\Zv_{a \leftarrow t} \leftarrow \mathrm{Decode}\left(\Zv_a \mid f^\mathrm{enc}_\theta(\Xv_t, L_a)\right)$
            \STATE Agreement log-probabilities: \\
            \qquad $\ell_{a \leftarrow s}^t \leftarrow \log \prob[\theta]{\Zv_{a \leftarrow s} \mid \Xv_t}$ \\
            \qquad $\ell_{a \leftarrow t}^s \leftarrow \log \prob[\theta]{\Zv_{a \leftarrow t} \mid \Xv_s}$
            \STATE Apply stop-gradients to supervised $\ell_{a \leftarrow s}^t$ and $\ell_{a \leftarrow t}^s$
            \STATE Total loss: $\Lc^\mathrm{total}(\theta) \leftarrow \Lc^\mathrm{sup}(\theta) + \gamma (\ell_{a \leftarrow s}^t + \ell_{a \leftarrow t}^s)$
            \STATE Update: $\theta \leftarrow \mathrm{optimizer\_update}(\Lc^\mathrm{total}, \theta)$
         \ENDWHILE
         \OUTPUT $\theta$
     \end{algorithmic}
     \vspace{-1ex}
 \end{algorithm}

\begin{figure*}[t]
    \centering
    \includegraphics[width=\textwidth]{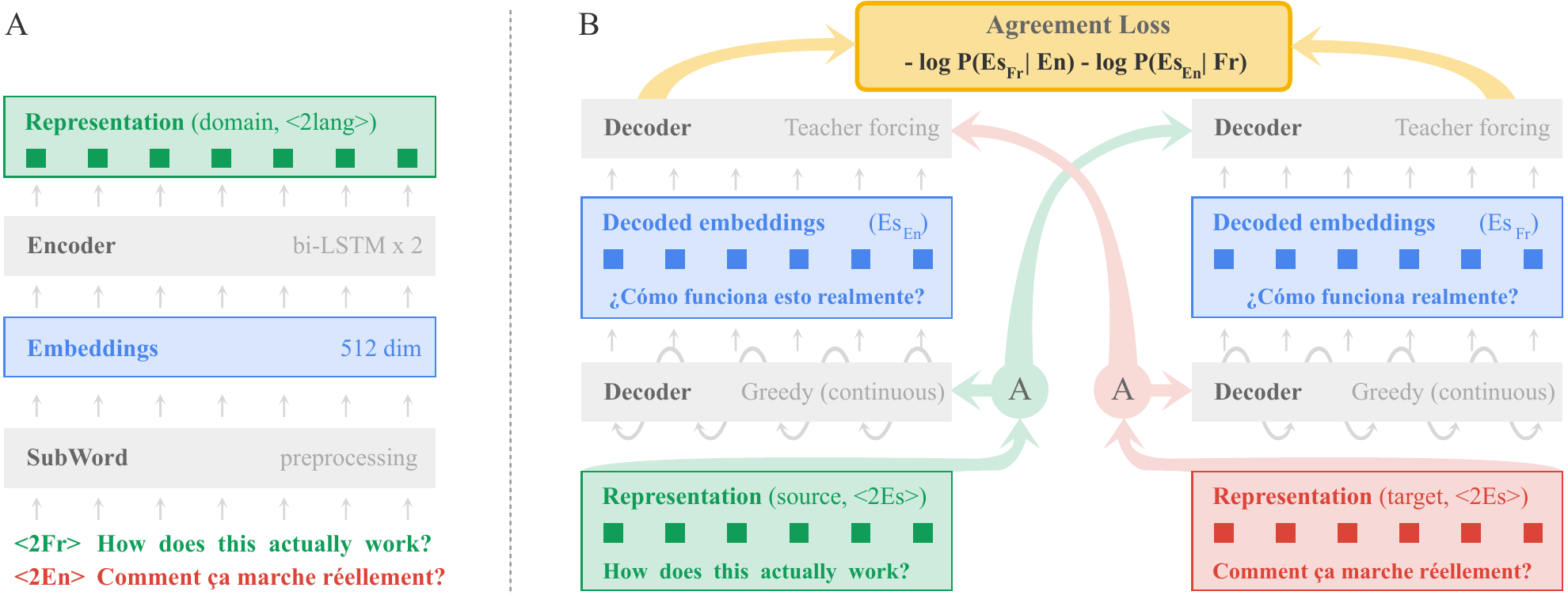}
    \vspace{-2ex}
    \caption{%
    \textbf{A.} Computation graph for the encoder.
    The representations depend on the input sequence and the target language tag.
    \textbf{B.} Computation graph for the agreement loss.
    First, encode source and target sequences with the auxiliary language tags.
    Next, decode $\zv_\Es$ from both $\xv_\En$ and $\xv_\Fr$ using continuous greedy decoder.
    Finally, evaluate log probabilities, $\log \prob[\theta]{\zv_\Es(\xv_\En) \mid \xv_\Fr}$ and $\log \prob[\theta]{\zv_\Es(\xv_\Fr) \mid \xv_\En}$, and compute a sample estimate of the agreement loss.}
    \label{fig:agreement-computation-graph}
    \vspace{-1ex}
\end{figure*}

\subsection{Agreement-based learning algorithm}
\label{sec:agreement-based-algorithm}

Having derived a new objective function \eqref{eq:agreement-multilingual-nmt-objective-ij}, we can now learn consistent multilingual NMT models using stochastic gradient method with a couple of extra tricks (Algorithm~\ref{alg:agreement-algorithm}).
The computation graph for the agreement loss is given in Figure~\ref{fig:agreement-computation-graph}.

\paragraph{Subsampling auxiliary languages.}
Computing agreement over \emph{all} languages for each pair of sentences at training time would be quite computationally expensive (to agree on $k$ translations, we would need to encode-decode the source and target sequences $k$ times each).
However, since the agreement lower bound is a sum over expectations \eqref{eq:agreement-term-lower-bound}, we can approximate it by subsampling:
at each training step (and for each sample in the mini-batch), we pick an auxiliary language uniformly at random and compute stochastic approximation of the agreement lower bound \eqref{eq:agreement-term-lower-bound} for that language only.
This stochastic approximation is simple, unbiased, and reduces per step computational overhead for the agreement term from $O(k)$ to $O(1)$.\footnote{In practice, note that there is still a constant factor overhead due to extra encoding-decoding steps to/from auxiliary languages, which is about $\times4$ when training on a single GPU. Parallelizing the model across multiple GPUs would easily compensate this overhead.}

\paragraph{Overview of the agreement loss computation.}
Given a pair of parallel sentences, $\xv_\En$ and $\xv_\Fr$, and an auxiliary language, say \Es, an estimate of the lower bound on the agreement term \eqref{eq:agreement-term-lower-bound} is computed as follows.
First, we concatenate {\Es} language tags to both $\xv_\En$ and $\xv_\Fr$ and encode the sequences so that both can be translated into {\Es} (the encoding process is depicted in Figure~\ref{fig:agreement-computation-graph}A).
Next, we decode each of the encoded sentences and obtain auxiliary translations, $\zv_\Es(\xv_\En)$ and $\zv_\Es(\xv_\Fr)$, depicted as blue blocks in Figure~\ref{fig:agreement-computation-graph}B.
Note that we now can treat pairs $(\xv_\Fr, \zv_\Es(\xv_\En))$ and $(\xv_\En, \zv_\Es(\xv_\Fr))$ as new parallel data for $\En \rightarrow \Es$ and $\Fr \rightarrow \Es$.

Finally, using these pairs, we can compute two log-probability terms (Figure~\ref{fig:agreement-computation-graph}B):
\begin{equation}
    \label{eq:agreement-terms}
    \begin{aligned}
        \log \prob[\theta]{\zv_\Es (\xv_\Fr) \mid \xv_\En} \\
        \log \prob[\theta]{\zv_\Es (\xv_\En) \mid \xv_\Fr}
    \end{aligned}
\end{equation}
using encoding-decoding with teacher forcing (same way as typically done for the supervised directions).
Crucially, note that $\zv_\Es(\xv_\En)$ corresponds to a supervised direction, $\En \rightarrow \Es$, while $\zv_\Es(\xv_\Fr)$ corresponds to zero-shot, $\Fr \rightarrow \Es$.
We want each of the components to (i) improve the zero-shot direction while (ii) minimally affecting the supervised direction.
To achieve (i), we use continuous decoding, and for (ii) we use stop-gradient-based protection of the supervised directions.
Both techniques are described below.

\paragraph{Greedy continuous decoding.}
In order to make $\zv_\Es (\xv_\En)$ and $\zv_\Es (\xv_\Fr)$ differentiable with respect to $\theta$ (hence, \emph{continuous} decoding), at each decoding step $t$, we treat the output of the RNN, $\hv^t$, as the key and use dot-product attention over the embedding vocabulary, $\Vv$, to construct $\zv_\Es^t$:
\begin{equation}
    \label{eq:continuous-decoder-output}
    \zv_\Es^{t} := \softmax\left\{(\hv^t)^\top\Vv\right\} \Vv
\end{equation}
In other words, auxiliary translations, $\zv_\Es (\xv_\En)$ and $\zv_\Es (\xv_\Fr)$, are fixed length sequences of differentiable embeddings computed in a greedy fashion.

\paragraph{Protecting supervised directions.}
Algorithm~\ref{alg:agreement-algorithm} scales agreement losses by a small coefficient $\gamma$.
We found experimentally that training could be sensitive to this hyperparameter since the agreement loss also affects the supervised sub-models.
For example, agreement of $\En \rightarrow \Es$ (supervised) and $\Fr \rightarrow \Es$ (zero-shot) may push the former towards a worse translation, especially at the beginning of training.
To stabilize training, we apply the \texttt{stop\_gradient} operator to the log probabilities and samples produced by the supervised sub-models before computing the agreement terms \eqref{eq:agreement-terms}, to zero-out the corresponding gradient updates.
\section{Experiments}
\label{sec:experiments}

We evaluate agreement-based training against baselines from the literature on three public datasets that have multi-parallel \emph{evaluation data} that allows assessing zero-shot performance.
We report results in terms of the BLEU score~\citep{papineni2002bleu} that was computed using \texttt{mteval-v13a.perl}.

\subsection{Datasets}

\paragraph{UN corpus.}
Following the setup introduced in \citet{sestorain2018zero}, we use two datasets, \emph{UNcorpus-1} and \emph{UNcorpus-2}, derived from the United Nations Parallel Corpus \citep{ziemski2016united}.
\emph{UNcorpus-1} consists of data in 3 languages, $\En$, $\Es$, $\Fr$, where \emph{UNcorpus-2} has $\Ru$ as the 4th language.
For training, we use parallel corpora between $\En$ and the rest of the languages, each about 1M sentences, sub-sampled from the official training data in a way that ensures no multi-parallel training data.
The \emph{dev} and \emph{test} sets contain 4,000 sentences and are all multi-parallel.

\paragraph{Europarl v7\footnote{\url{http://www.statmt.org/europarl/}}.}
We consider the following languages: $\De$, $\En$, $\Es$, $\Fr$.
For training, we use parallel data between $\En$ and the rest of the languages (about 1M sentences per corpus), preprocessed to avoid multi-parallel sentences, as was also done by \citet{cheng2017joint} and \citet{chen2017teacher} and described below.
The \emph{dev} and \emph{test} sets contain 2,000 multi-parallel sentences.

\paragraph{IWSLT17\footnote{\url{https://sites.google.com/site/iwsltevaluation2017/TED-tasks}}.}
We use data from the official multilingual task:
5 languages ($\De$, $\En$, $\It$, $\Nl$, $\Ro$), 20 translation tasks of which 4 zero-shot ($\De \leftrightarrow \Nl$ and $\It \leftrightarrow \Ro$) and the rest 16 supervised.
Note that this dataset has a significant overlap between parallel corpora in the supervised directions (up to 100K sentence pairs per direction).
This implicitly makes the dataset multi-parallel and defeats the purpose of zero-shot evaluation \citep{dabre2017kyoto}.
To avoid spurious effects, we also derived \textbf{IWSLT17}$^\star$ dataset from the original one by restricting supervised data to only $\En \leftrightarrow \{\De, \Nl, \It, \Ro\}$ and removing overlapping pivoting sentences.
We report results on both the official and preprocessed datasets.

\paragraph{Preprocessing.}
To properly evaluate systems in terms of zero-shot generalization, we preprocess Europarl and IWSLT$^\star$ to avoid multi-lingual parallel sentences of the form \emph{source-pivot-target}, where \emph{source-target} is a zero-shot direction.
To do so, we follow \citet{cheng2017joint,chen2017teacher} and randomly split the overlapping pivot sentences of the original \emph{source-pivot} and \emph{pivot-target} corpora into two parts and merge them separately with the non-overlapping parts for each pair.
Along with each parallel training sentence, we save information about source and target tags, after which all the data is combined and shuffled.
Finally, we use a shared multilingual subword vocabulary \citep{sennrich2015neural} on the training data (with 32K merge ops), separately for each dataset.
Data statistics are provided in Appendix~\ref{app:data-details}.

\subsection{Training and evaluation}

Additional details on the hyperparameters can be found in Appendix~\ref{app:model-details}.

\paragraph{Models.}
We use a smaller version of the GNMT architecture \citep{wu2016google} in all our experiments: 512-dimensional embeddings (separate for source and target sides), 2 bidirectional LSTM layers of 512 units each for encoding, and GNMT-style, 4-layer, 512-unit LSMT decoder with residual connections from the 2nd layer onward.

\paragraph{Training.}
We trained the above model using the \emph{standard method} of \citet{johnson2016google} and using our proposed \emph{agreement-based} training (Algorithm~\ref{alg:agreement-algorithm}).
In both cases, the model was optimized using Adafactor \citep{shazeer2018adafactor} on a machine with 4 P100 GPUs for up to 500K steps, with early stopping on the dev set.

\begin{table}[t]
\centering
\scriptsize
\def\arraystretch{0.9}
\setlength{\tabcolsep}{4.8pt}
\begin{tabular}{@{}lrrr|rr|r@{}}
\toprule
                        & \multicolumn{3}{c|}{\citet{sestorain2018zero}$^\dagger$}  & \multicolumn{2}{c|}{Our baselines}  \\
\cmidrule[0.5pt]{2-6}
                        & PBSMT & NMT-0 & Dual-0 & Basic & Pivot & Agree  \\
\cmidrule[0.5pt]{1-7}
$\En \rightarrow \Es$   & \textcolor{gray}{61.26} & 51.93 & ---    & 56.58 & \textcolor{gray}{56.58} & 56.36 \\
$\En \rightarrow \Fr$   & \textcolor{gray}{50.09} & 40.56 & ---    & 44.27 & \textcolor{gray}{44.27} & 44.80 \\
$\Es \rightarrow \En$   & \textcolor{gray}{59.89} & 51.58 & ---    & 55.70 & \textcolor{gray}{55.70} & 55.24 \\
$\Fr \rightarrow \En$   & \textcolor{gray}{52.22} & 43.33 & ---    & 46.46 & \textcolor{gray}{46.46} & 46.17 \\
\cmidrule[0.5pt]{1-7}
Supervised (avg.)       & \textcolor{gray}{55.87} & 46.85 & ---    & 50.75 & \textcolor{gray}{50.75} & 50.64 \\
\cmidrule[0.5pt]{1-7}
$\Es \rightarrow \Fr$   & \textcolor{gray}{52.44} & 20.29 & 36.68  & 34.75 & \textbf{38.10} & 37.54 \\
$\Fr \rightarrow \Es$   & \textcolor{gray}{49.79} & 19.01 & 39.19  & 37.67 & \textbf{40.84} & 40.02 \\
\cmidrule[0.5pt]{1-7}
Zero-shot (avg.)        & \textcolor{gray}{51.11} & 19.69 & 37.93  & 36.21 & \textbf{39.47} & 38.78 \\
\bottomrule
\end{tabular}
$^\dagger$Source: \url{https://openreview.net/forum?id=ByecAoAqK7}.
\vspace{-1ex}
\caption{Results on UNCorpus-1.}
\label{tab:uncorpus-exp1}
\vspace{3ex}
\centering
\scriptsize
\def\arraystretch{0.9}
\setlength{\tabcolsep}{4.8pt}
\begin{tabular}{@{}lrrr|rr|r@{}}
\toprule
                        & \multicolumn{3}{c|}{\citet{sestorain2018zero}}  & \multicolumn{2}{c|}{Our baselines}  \\
\cmidrule[0.5pt]{2-6}
                        & PBSMT & NMT-0 & Dual-0 & Basic & Pivot & Agree  \\
\midrule
$\En \rightarrow \Es$   & \textcolor{gray}{61.26} & 47.51 & 44.30  & 55.15 & \textcolor{gray}{55.15}  & 54.30 \\
$\En \rightarrow \Fr$   & \textcolor{gray}{50.09} & 36.70 & 34.34  & 43.42 & \textcolor{gray}{43.42}  & 42.57 \\
$\En \rightarrow \Ru$   & \textcolor{gray}{43.25} & 30.45 & 29.47  & 36.26 & \textcolor{gray}{36.26}  & 35.89 \\
$\Es \rightarrow \En$   & \textcolor{gray}{59.89} & 48.56 & 45.55  & 54.35 & \textcolor{gray}{54.35}  & 54.33 \\
$\Fr \rightarrow \En$   & \textcolor{gray}{52.22} & 40.75 & 37.75  & 45.55 & \textcolor{gray}{45.55}  & 45.87 \\
$\Ru \rightarrow \En$   & \textcolor{gray}{52.59} & 39.35 & 37.96  & 45.52 & \textcolor{gray}{45.52}  & 44.67 \\
\midrule
Supervised (avg.)       & \textcolor{gray}{53.22} & 40.55 & 36.74  & 46.71 & \textcolor{gray}{46.71}  & 46.27 \\
\midrule
$\Es \rightarrow \Fr$   & \textcolor{gray}{52.44} & 25.85 & 34.51 & 34.73 & 35.93             & \textbf{36.02}  \\
$\Fr \rightarrow \Es$   & \textcolor{gray}{49.79} & 22.68 & 37.71 & 38.20 & 39.51             & \textbf{39.94} \\
$\Es \rightarrow \Ru$   & \textcolor{gray}{39.69} &  9.36 & 24.55 & 26.29 & 27.15             & \textbf{28.08} \\
$\Ru \rightarrow \Es$   & \textcolor{gray}{49.61} & 26.26 & 33.23 & 33.43 & \textbf{37.17}    & 35.01 \\
$\Fr \rightarrow \Ru$   & \textcolor{gray}{36.48} &  9.35 & 22.76 & 23.88 & 24.99             & \textbf{25.13} \\
$\Ru \rightarrow \Fr$   & \textcolor{gray}{43.37} & 22.43 & 26.49 & 28.52 & \textbf{30.06}    & 29.53 \\
\midrule
Zero-shot (avg.)        & \textcolor{gray}{45.23} & 26.26 & 29.88 & 30.84 & \textbf{32.47}    & 32.29 \\
\bottomrule
\end{tabular}
\caption{Results on UNCorpus-2.}
\label{tab:uncorpus-exp2}
\end{table}
\begin{table}[t]
\centering
\scriptsize
\def\arraystretch{0.8}
\setlength{\tabcolsep}{8.25pt}
\begin{tabular}{@{}lrr|rr|r@{}}
\toprule
                        & \multicolumn{2}{c|}{Previous work}            & \multicolumn{2}{c|}{Our baselines}  \\
\cmidrule[0.5pt]{2-5}
                        & Soft$^\ddagger$       & Distill$^\dagger$     & Basic & Pivot                   & Agree  \\
\cmidrule[0.5pt]{1-6}
$\En \rightarrow \Es$   & ---                   & ---                   & 34.69 & \textcolor{gray}{34.69} & 33.80 \\
$\En \rightarrow \De$   & ---                   & ---                   & 23.06 & \textcolor{gray}{23.06} & 22.44 \\
$\En \rightarrow \Fr$   & 31.40                 & ---                   & 33.87 & \textcolor{gray}{33.87} & 32.55 \\
$\Es \rightarrow \En$   & 31.96                 & ---                   & 34.77 & \textcolor{gray}{34.77} & 34.53 \\
$\De \rightarrow \En$   & 26.55                 & ---                   & 29.06 & \textcolor{gray}{29.06} & 29.07 \\
$\Fr \rightarrow \En$   & ---                   & ---                   & 33.67 & \textcolor{gray}{33.67} & 33.30 \\
\cmidrule[0.5pt]{1-6}
Supervised (avg.)       & ---                   & ---                   & 31.52 & \textcolor{gray}{31.52} & 30.95 \\
\cmidrule[0.5pt]{1-6}
$\Es \rightarrow \De$   & ---                   & ---                   & 18.23 & 20.14                   & \textbf{20.70} \\
$\De \rightarrow \Es$   & ---                   & ---                   & 20.28 & \textbf{26.50}          & 22.45 \\
$\Es \rightarrow \Fr$   & 30.57                 & \textbf{33.86}        & 27.99 & 32.56                   & 30.94 \\
$\Fr \rightarrow \Es$   & ---                   & ---                   & 27.12 & \textbf{32.96}          & 29.91 \\
$\De \rightarrow \Fr$   & 23.79                 & \textbf{27.03}        & 21.36 & 25.67                   & 24.45 \\
$\Fr \rightarrow \De$   & ---                   & ---                   & 18.57 & \textbf{19.86}          & 19.15 \\
\cmidrule[0.5pt]{1-6}
Zero-shot (avg.)        & ---                   & ---                   & 22.25 & 26.28                   & 24.60 \\
\bottomrule
\end{tabular}\\
$^\dagger$Soft pivoting \citep{cheng2017joint}. $^\ddagger$Distillation \citep{chen2017teacher}.
\vspace{-1ex}
\caption{Zero-shot results on Europarl.
Note that \textit{Soft} and \textit{Distill} are not multilingual systems.}
\label{tab:europarl}
\end{table}
\begin{table}[t]
\centering
\scriptsize
\def\arraystretch{0.9}
\setlength{\tabcolsep}{8.2pt}
\begin{tabular}{@{}lrr|rr|r@{}}
\toprule
                        & \multicolumn{2}{c|}{Previous work}            & \multicolumn{2}{c|}{Our baselines}  \\
\cmidrule[0.5pt]{2-5}
                        & SOTA$^\dagger$        & CPG$^\ddagger$        & Basic & Pivot                     & Agree  \\
\cmidrule[0.5pt]{1-6}
Supervised (avg.)       & 24.10                 & 19.75                 & 24.63 & \textcolor{gray}{24.63}   & 23.97 \\
Zero-shot  (avg.)       & 20.55                 & 11.69                 & 19.86 & 19.26                     & \textbf{20.58} \\
\bottomrule
\end{tabular}\\
$^\dagger$Table 2 from \citet{dabre2017kyoto}. $^\ddagger$Table 2 from \citet{platanios2018contextual}.
\vspace{-1ex}
\caption{Results on the official IWSLT17 multilingual task.}
\label{tab:iwslt17}
\vspace{3ex}
\scriptsize
\def\arraystretch{0.9}
\begin{tabular}{@{}lrr|r@{}}
\toprule
                        & Basic & Pivot                   & Agree  \\
\cmidrule[0.5pt]{1-4}
Supervised (avg.)       & 28.72 & \textcolor{gray}{28.72} & \textbf{29.17} \\
Zero-shot  (avg.)       & 12.61 & \textbf{17.68}                   & 15.23 \\
\bottomrule
\end{tabular}
\caption{Results on our proposed IWSLT17$^\star$.}
\label{tab:iwslt17-star}
\end{table}

\paragraph{Evaluation.}
We focus our evaluation mainly on zero-shot performance of the following methods:\\
(a) \underline{\texttt{Basic}}, which stands for directly evaluating a multilingual GNMT model after standard training~\cite{johnson2016google}.\\
(b) \underline{\texttt{Pivot}}, which performs pivoting-based inference using a multilingual GNMT model (after standard training); often regarded as gold-standard.\\
(c) \underline{\texttt{Agree}}, which applies a multilingual GNMT model trained with agreement losses directly to zero-shot directions.

To ensure a fair comparison in terms of model capacity, all the techniques above use the same multilingual GNMT architecture described in the previous section.
All other results provided in the tables are as reported in the literature.

\paragraph{Implementation.}
All our methods were implemented using TensorFlow \citep{abadi2016tensorflow} on top of tensor2tensor library \citep{vaswani2018tensor2tensor}.
Our code will be made publicly available.\footnote{\url{www.cs.cmu.edu/~mshediva/code/}}

\subsection{Results on UN Corpus and Europarl}

\paragraph{UN Corpus.}
Tables~\ref{tab:uncorpus-exp1} and~\ref{tab:uncorpus-exp2}
show results on the UNCorpus datasets.
Our approach consistently outperforms \texttt{Basic} and \texttt{Dual-0}, despite the latter being trained with additional monolingual data~\citep{sestorain2018zero}.
We see that models trained with agreement perform comparably to \texttt{Pivot}, outperforming it in some cases, \eg, when the target is Russian, perhaps because it is quite different linguistically from the English pivot.

Furthermore, unlike~\texttt{Dual-0}, \texttt{Agree} maintains high performance in the supervised directions (within 1 BLEU point compared to \texttt{Basic}), indicating that our agreement-based approach is effective as a part of a single multilingual system.

\paragraph{Europarl.}
Table~\ref{tab:europarl} shows the results on the Europarl corpus.
On this dataset, our approach consistently outperforms \texttt{Basic} by 2-3 BLEU points but lags a bit behind \texttt{Pivot} on average (except on $\Es \rightarrow \De$ where it is better).
\citet{cheng2017joint}\footnote{We only show their best zero-resource result in the table since some of their methods require direct parallel data.} and \citet{chen2017teacher} have reported zero-resource results on a subset of these directions and our approach outperforms the former but not the latter on these pairs.
Note that both \citet{cheng2017joint} and \citet{chen2017teacher} train separate models for each language pair and the approach of~\citet{chen2017teacher} would require training $O(k^2)$ models to encompass all the pairs.
In contrast, we use a single multilingual architecture which has more limited model capacity (although in theory, our approach is also compatible with using separate models for each direction).

\begin{figure}[t]
    \centering
    \includegraphics[width=\columnwidth]{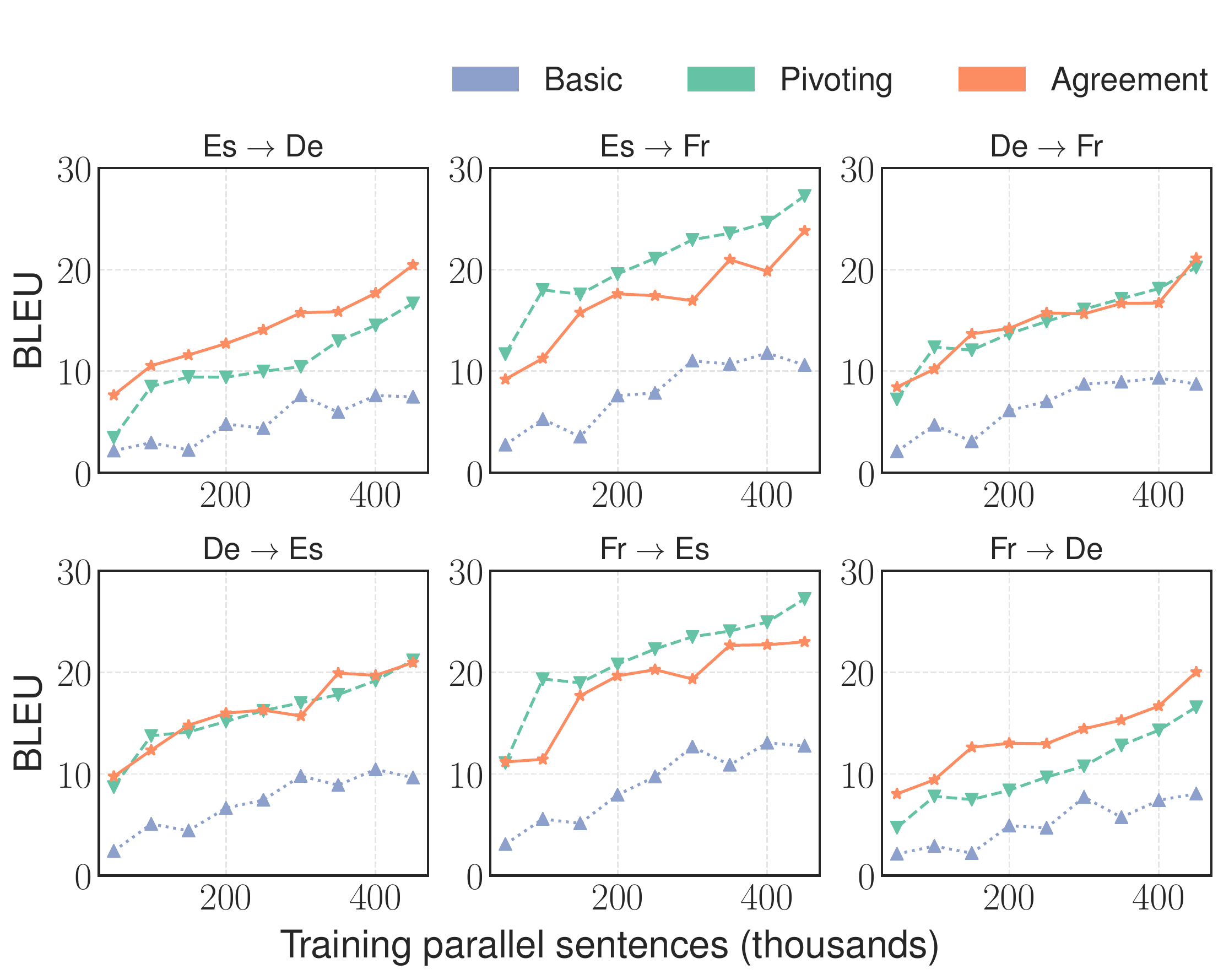}
    \caption{BLEU on the dev set for \texttt{Agree} and the baselines trained on smaller subsets of the Europarl corpus.}
    \label{fig:bleu-vs-data-size}
\end{figure}

\subsection{Analysis of IWSLT17 zero-shot tasks}

Table~\ref{tab:iwslt17} presents results on the original IWSLT17 task.
We note that because of the large amount of data overlap and presence of many supervised translation pairs (16) the vanilla training method \citep{johnson2016google} achieves very high zero shot performance, even outperforming \texttt{Pivot}.
While our approach gives small gains over these baselines, we believe the dataset's pecularities make it not reliable for evaluating zero-shot generalization.

On the other hand, on our proposed preprocessed IWSLT17$^\star$ that eliminates the overlap and reduces the number of supervised directions (8), there is a considerable gap between the supervised and zero-shot performance of \texttt{Basic}.
\texttt{Agree} performs  better than \texttt{Basic} and is slightly worse than \texttt{Pivot}.

\subsection{Small data regime}
To better understand the dynamics of different methods in the small data regime, we also trained all our methods on subsets of the Europarl for 200K steps and evaluated on the dev set.
The training set size varied from 50 to 450K parallel sentences.
From Figure~\ref{fig:bleu-vs-data-size}, \texttt{Basic} tends to perform extremely poorly while \texttt{Agree} is the most robust (also in terms of variance across zero-shot directions).
We see that \texttt{Agree} generally upper-bounds \texttt{Pivot}, except for the $(\Es, \Fr)$ pair, perhaps due to fewer cascading errors along these directions.

\section{Conclusion}
\label{sec:conclusion}

In this work, we studied zero-shot generalization in the context of multilingual neural machine translation.
First, we introduced the concept of zero-shot consistency that implies generalization.
Next, we proposed a provably consistent agreement-based learning approach for zero-shot translation.
Empirical results on three datasets showed that agreement-based learning results in up to +3 BLEU zero-shot improvement over the \citet{johnson2016google} baseline, compares favorably to other approaches in the literature \citep{cheng2017joint, sestorain2018zero}, is competitive with pivoting, and does not lose in performance on supervised directions.

We believe that the theory and methodology behind agreement-based learning could be useful beyond translation, especially in multi-modal settings.
For instance, it could be applied to tasks such as cross-lingual natural language inference~\cite{conneau2018xnli}, style-transfer~\citep{shen2017style, fu2017style, prabhumoye2018style}, or multilingual image or video captioning.
Another interesting future direction would be to explore different hand-engineered or learned data representations, which one could use to encourage models to agree on during training (\eg, make translation models agree on latent semantic parses, summaries, or potentially other data representations available at training time).

\section*{Acknowledgments}

We thank Ian Tenney and Anthony Platanios for many insightful discussions, Emily Pitler for the helpful comments on the early draft of the paper, and anonymous reviewers for careful reading and useful feedback.

\bibliography{references}
\bibliographystyle{acl_natbib}

\appendix
\clearpage
\section{Appendices}
\label{app:appendix}

\subsection{Complete likelihood}
\label{app:complete-likelihood}

Given a set of conditional models, $\{\prob[\theta]{\xv_j \mid \xv_i}\}$, we can write out the full likelihood over equivalent translations, $(\xv_1, \dots, \xv_k)$, as follows:
\begin{equation}
    \label{eq:poe}
    \begin{split}
        \prob[\theta]{\xv_1, \dots, \xv_k} := \frac{1}{Z} \prod_{i, j \in \Ec} \prob[\theta]{\xv_j \mid \xv_i}
    \end{split}
\end{equation}
where $Z := \sum_{\xv_1, \dots, \xv_k} \prod_{i, j \in \Ec} \prob[\theta]{\xv_j \mid \xv_i}$ is the normalizing constant and $\Ec$ denotes \emph{all} edges in the graph (Figure~\ref{fig:graphical-model}).
Given only bilingual parallel corpora, $\Cc_{ij}$ for $i, j \in \Ec_s$, we can observe only certain pairs of variables.
Therefore, the log-likelihood of the data can be written as:
\begin{equation}
    \label{eq:complete-likelihood}
    \begin{split}
        \MoveEqLeft \Lc(\theta) := \\
        & \sum_{i, j \in \Ec_s} \sum_{\xv_i, \xv_j \in \Cc_{ij}} \log \sum_{\zv} \prob[\theta]{\xv_1, \dots, \xv_k} \\
    \end{split}
\end{equation}
Here, the outer sum iterates over available corpora.
The middle sum iterates over parallel sentences in a corpus.
The most inner sum marginalizes out unobservable sequences, denoted $\zv := \{\xv_l\}_{l \neq i, j}$, which are sentences equivalent under this model to $\xv_i$ and $\xv_j$ in languages other than $L_i$ and $L_j$.
Note that due to the inner-most summation, computing the log-likelihood is intractable.

We claim the following.
\begin{claim}
Maximizing the full log-likelihood yields zero-shot consistent models (Definition~\ref{def:zero-shot-consistency}).
\end{claim}
\begin{proof}
To better understand why this is the case, let us consider example in Figure~\ref{fig:graphical-model} and compute the log-likelihood of $(\xv_1, \xv_2)$:
\begin{equation*}
    \begin{split}
        \MoveEqLeft \log \prob[\theta]{\xv_1, \xv_2} \\
        = & \, \log \sum\nolimits_{\xv_3, \xv_4} \prob[\theta]{\xv_1, \xv_2, \xv_3, \xv_4} \\
        \propto & \, \log \prob[\theta]{\xv_1 \mid \xv_2} + \log \prob[\theta]{\xv_2 \mid \xv_1} + \\
        & \log \sum\nolimits_{\xv_3, \xv_4} \prob[\theta]{\xv_1 \mid \xv_3} \textcolor{goog-green}{\prob[\theta]{\xv_3 \mid \xv_1}} \times \\
        & \qquad\qquad\quad\, \prob[\theta]{\xv_2 \mid \xv_3} \textcolor{goog-green}{\prob[\theta]{\xv_3 \mid \xv_2}} \times \\
        & \qquad\qquad\quad\, \prob[\theta]{\xv_1 \mid \xv_4} \textcolor{goog-blue}{\prob[\theta]{\xv_4 \mid \xv_1}} \times \\
        & \qquad\qquad\quad\, \prob[\theta]{\xv_2 \mid \xv_4} \textcolor{goog-blue}{\prob[\theta]{\xv_4 \mid \xv_2}} \times \\
        & \qquad\qquad\quad\, \prob[\theta]{\xv_3 \mid \xv_4} \prob[\theta]{\xv_4 \mid \xv_3}
    \end{split}
\end{equation*}
Note that the terms that encourage agreement on the translation into $L_3$ are colored in \textcolor{goog-green}{green} (similarly, terms that encourage agreement on the translation into $L_4$ are colored in \textcolor{goog-blue}{blue}).
Since all other terms are probabilities and bounded by 1, we have:

\begin{equation*}
    \begin{split}
        \MoveEqLeft \log \prob[\theta]{\xv_1, \xv_2} + \log Z \\
        \leq & \, \log \prob[\theta]{\xv_1 \mid \xv_2} + \log \prob[\theta]{\xv_2 \mid \xv_1} + \\
        & \log \sum\nolimits_{\xv_3, \xv_4} \textcolor{goog-green}{\prob[\theta]{\xv_3 \mid \xv_1}} \textcolor{goog-green}{\prob[\theta]{\xv_3 \mid \xv_2}} \times \\
        & \qquad\qquad\quad\,
        \textcolor{goog-blue}{\prob[\theta]{\xv_4 \mid \xv_1}} \textcolor{goog-blue}{\prob[\theta]{\xv_4 \mid \xv_2}} \\
        \equiv & \, \Lc^\mathrm{agree}(\theta)
    \end{split}
\end{equation*}
In other words, the full log likelihood lower-bounds the agreement objective (up to a constant $\log Z$).
Since optimizing for agreement leads to consistency (Theorem~\ref{thm:agreement-consistency}), and maximizing the full likelihood would necessarily improve the agreement, the claim follows.
\end{proof}

\begin{remark}
Note that the other terms in the full likelihood also have a non-trivial purpose: (a) the terms $\prob[\theta]{\xv_1 \mid \xv_3}$, $\prob[\theta]{\xv_1 \mid \xv_4}$, $\prob[\theta]{\xv_2 \mid \xv_3}$, $\prob[\theta]{\xv_2 \mid \xv_4}$, encourage the model to correctly reconstruct $\xv_1$ and $\xv_2$ when back-translating from unobserved languages, $L_3$ and $L_4$,
and (b) terms $\prob[\theta]{\xv_3 \mid \xv_4}$, $\prob[\theta]{\xv_4 \mid \xv_3}$ enforce consistency between the latent representations.
In other words, full likelihood accounts for a combination of agreement, back-translation, and latent consistency.
\end{remark}

\begin{figure}[t]
    \centering
    \begin{tikzpicture}[->, >=stealth', shorten >=1pt, auto, node distance=1cm, thick, main/.style={circle, draw=black!50, minimum width=15pt}]
        \node[main, fill=gray!20] (L1) at (0, 0) {$L_1$};
        \node[main, fill=gray!20] (L2) at (-2, -1) {$L_2$};
        \node[main] (L3) at (2, -1) {$L_3$};
        \node[main] (L4) at (0, 2.2) {$L_4$};
        
        \draw [-, line width=0.4mm] (L1) -- (L2) node[midway, above, sloped] {$\Cc_\text{12}$};
        \draw [-, line width=0.4mm, color=gray] (L1) -- (L3) node[midway, above, sloped] {$\Cc_\text{13}$};
        \draw [-, line width=0.4mm, color=gray] (L1) -- (L4) node[midway, above, sloped] {$\Cc_\text{14}$};
        
        \draw [-, dotted, color=red] (L2) -- (L4);
        \draw [-, dotted, color=red] (L2) -- (L3);
        \draw [-, dotted, color=red] (L3) -- (L4);
    \end{tikzpicture}
    \caption{Probabilistic graphical model for a multilingual system with four languages $(L_1, L_2, L_3, L_4)$.
    Variables can only be observed only in pairs (shaded in the graph).}
    \label{fig:graphical-model}
\end{figure}
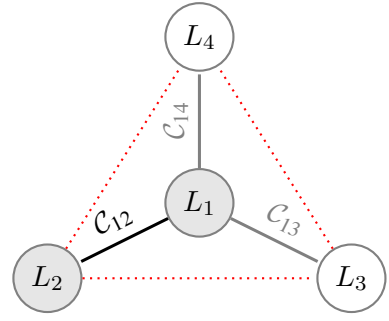

\subsection{Proof of agreement consistency}
\label{app:agreement-consistency-proof}

The statement of Theorem~\ref{thm:agreement-consistency} mentions an assumption on the true distribution of the equivalent translations.
The assumption is as follows.
\begin{assumption}
\label{asm:ground-truth-bounds}
Let $\prob{\xv_i \mid \xv_j, \xv_k}$ be the ground truth conditional distribution that specifies the probability of $\xv_i$ to be a translation of $\xv_j$ and $\xv_k$ into language $L_i$, given that $(\xv_j, \xv_k)$ are correct translations of each other in languages $L_j$ and $L_k$, respectively.
We assume:
\begin{equation*}
0\le \delta \le \ep[\xv_k \mid \xv_i, \xv_j]{\prob{\xv_i \mid \xv_j, \xv_k}} \le \xi \le 1
\end{equation*}
\end{assumption}
This assumption means that, even though there might be multiple equivalent translations, there must be not too many of them (implied by the $\delta$ lower bound) and none of them must be much more preferable than the rest (implied by the $\xi$ upper bound).
Given this assumption, we can prove the following simple lemma.
\begin{lemma}
\label{lem:expectation-bound}
Let $L_i \rightarrow L_j$ be one of the supervised directions, $\ep[\xv_i, \xv_j]{-\log \prob[\theta]{\xv_j \mid \xv_i}} \le \varepsilon$.
Then the following holds:
\begin{equation*}
\ep[\xv_i \mid \xv_j, \xv_k]{\frac{\prob[\theta]{\xv_j \mid \xv_i}}{\prob{\xv_j \mid \xv_i, \xv_k}}} \geq \log \frac{1}{\xi} - \varepsilon\delta
\end{equation*}
\end{lemma}
\begin{proof}
First, using Jensen's inequality, we have:
\begin{equation*}
\begin{split}
& \log \ep[\xv_i \mid \xv_j, \xv_k]{\frac{\prob[\theta]{\xv_j \mid \xv_i}}{\prob{\xv_j \mid \xv_i, \xv_k}}} \geq \\
& \ep[\xv_i \mid \xv_j, \xv_k]{\log \prob[\theta]{\xv_j \mid \xv_i} - \log \prob{\xv_j \mid \xv_i, \xv_k}}
\end{split}
\end{equation*}
The bound on the supervised direction implies that
\begin{equation*}
\ep[\xv_i \mid \xv_j, \xv_k]{- \log \prob[\theta]{\xv_j \mid \xv_i}} \geq -\varepsilon\delta
\end{equation*}
To bound the second term, we use Assumption~\ref{asm:ground-truth-bounds}:
\begin{equation*}
\ep[\xv_i \mid \xv_j, \xv_k]{- \log \prob{\xv_j \mid \xv_i, \xv_k}} \geq \log \frac{1}{\xi}
\end{equation*}
Putting these together yields the bound.
\end{proof}

\noindent Now, using Lemma~\ref{lem:expectation-bound}, we can prove Theorem~\ref{thm:agreement-consistency}.

\begin{proof}
By assumption, the agreement-based loss is bounded by $\varepsilon$.
Therefore, expected cross-entropy on all supervised terms, $L_1 \leftrightarrow L_2$, is bounded by $\varepsilon$.
Moreover, the agreement term (which is part of the objective) is also bounded:
\begin{equation*}
- \ep[\xv_i, \xv_j]{\sum_{\xv_k} \prob[\theta]{\xv_k \mid \xv_j} \log \prob[\theta]{\xv_k \mid \xv_i}} \le \varepsilon
\end{equation*}
Expanding this expectation, we have:
\begin{equation*}
\begin{split}
& \sum_{\xv_i, \xv_j} \prob{\xv_i, \xv_j} \sum_{\xv_k} \prob[\theta]{\xv_k \mid \xv_j} \log \prob[\theta]{\xv_k \mid \xv_i}] \\
& = \sum_{\xv_i, \xv_j, \xv_k} \prob{\xv_i, \xv_j, \xv_k} \times \\
& \qquad\quad\,\, \frac{\prob[\theta]{\xv_k \mid \xv_j}}{\prob{\xv_k \mid \xv_i, \xv_j}} \log \prob[\theta]{\xv_k \mid \xv_i} \\
& = \sum_{\xv_i, \xv_k} \ep[\xv_j \mid \xv_i, \xv_k]{\frac{\prob[\theta]{\xv_k \mid \xv_j}}{\prob{\xv_k \mid \xv_i, \xv_j}}} \times \\
& \qquad\quad\,\, \prob{\xv_i, \xv_k} \log \prob[\theta]{\xv_k \mid \xv_i}
\end{split}
\end{equation*}
Combining that with Lemma~\ref{lem:expectation-bound}, we have:
\begin{equation*}
\ep[\xv_i, \xv_k]{-\log \prob[\theta]{\xv_k \mid \xv_i}} \le \frac{\varepsilon}{\log\frac{1}{\xi} - \delta\varepsilon} \equiv \kappa(\varepsilon)
\end{equation*}
Since by Assumption~\ref{asm:ground-truth-bounds}, $\delta$ and $\xi$ are some constants, $\kappa(\varepsilon) \rightarrow 0$ as $\varepsilon \rightarrow 0$.
\end{proof}

\subsection{Consistency of distillation and pivoting}
\label{app:distillation-pivoting-consistency}

As we mentioned in the main text of the paper, distillation \citep{chen2017teacher} and pivoting yield zero-shot consistent models.
Let us understand why this is the case.

In our notation, given $L_1 \rightarrow L_2$ and $L_2 \rightarrow L_3$ as supervised directions, distillation optimizes a KL-divergence between $\prob[\theta]{\xv_3 \mid \xv_2}$ and $\prob[\theta]{\xv_3 \mid \xv_1}$, where the latter is a zero-shot model and the former is supervised.
Noting that KL-divergence lower-bounds cross-entropy, it is a loser bound on the agreeement loss.
Hence, by ensuring that KL is low, we also ensure that the models agree, which implies consistency (a more formal proof would exactly follow the same steps as the proof of Theorem~\ref{thm:agreement-consistency}).

To prove consistency of pivoting, we need an additional assumption on the quality of the source-pivot model.

\begin{assumption}
\label{asm:source-pivot-model-quality}
Let $\prob[\theta]{\xv_j \mid \xv_i}$ be the source-pivot model.
We assume the following bound holds for each pair of equivalent translations, $(\xv_j, \xv_k)$:
\begin{equation*}
\ep[\xv_i \mid \xv_j, \xv_k]{\frac{\prob[\theta]{\xv_j \mid \xv_i}}{\prob{\xv_j \mid \xv_i, \xv_k}}} \leq C
\end{equation*}
where $C > 0$ is some constant.
\end{assumption}

\begin{theorem}[Pivoting consistency]
\label{thm:pivoting-consistency}
Given the conditions of Theorem~\ref{thm:agreement-consistency} and Assumption~\ref{asm:source-pivot-model-quality}, pivoting is zero-shot consistent.
\end{theorem}

\begin{proof}
We can bound the expected error on pivoting as follows (using Jensen's inequality and the conditions from our assumptions):
\begin{align*}
\MoveEqLeft \ep[\xv_i, \xv_k]{- \log \sum_{\xv_j} \prob[\theta]{\xv_j \mid \xv_i} \prob[\theta]{\xv_k \mid \xv_j}} \\
& \leq \ep[\xv_i, \xv_j, \xv_k]{- \prob[\theta]{\xv_j \mid \xv_i} \log \prob[\theta]{\xv_k \mid \xv_j}} \\
& \leq \sum_{\xv_i, \xv_k} \ep[\xv_j \mid \xv_i, \xv_k]{\frac{\prob[\theta]{\xv_k \mid \xv_j}}{\prob{\xv_k \mid \xv_i, \xv_j}}} \times \\
& \qquad\quad\,\, \prob{\xv_i, \xv_k} \log \prob[\theta]{\xv_k \mid \xv_i} \\
& \leq C \varepsilon
\end{align*}
\end{proof}

\subsection{Details on the models and training}
\label{app:model-details}

\paragraph{Architecture.}
All our NMT models used the GNMT \citep{wu2016google} architecture with  Luong attention~\citep{luong2015effective}, 2 bidirectional encoder, and 4-layer decoder with residual connections.
All hidden layers (including embeddings) had 512 units.
Additionally, we used separate embeddings on the encoder and decoder sides as well as tied weights of the softmax that produced logits with the decoder-side (\ie, target) embeddings.
Standard dropout of 0.2 was used on all hidden layers.
Most of the other hyperparameters we set to default in the T2T \citep{vaswani2018tensor2tensor} library for the text2text type of problems.

\paragraph{Training and hyperparameters.}
We scaled agreement terms in the loss by $\gamma = 0.01$.
The training was done using Adafactor~\citep{shazeer2018adafactor} optimizer with 10,000 burn-in steps at 0.01 learning rate and further standard square root decay (with the default settings for the decay from the T2T library).
Additionally, implemented agreement loss as a subgraph as a loss was not computed if $\gamma$ was set to 0.
This allowed us to start training multilingual NMT models in the burn-in mode using the composite likelihood objective and then switch on agreement starting some point during optimization (typically, after the first 100K iterations; we also experimented with 0, 50K, 200K, but did not notice any difference in terms of final performance).
Since the agreement subgraph was not computed during the initial training phase, it tended to accelerate training of agreement models.

\subsection{Details on the datasets}
\label{app:data-details}

Statistics of the IWSLT17 and IWSLT17$^\star$ datasets are summarized in Table~\ref{tab:dataset-statistics}.
UNCorpus and and Europarl datasets were exactly as described by \citet{sestorain2018zero} and \citet{chen2017teacher,cheng2017joint}, respectively.

\begin{table}[t]
\centering
\scriptsize
\def\arraystretch{1.2}
\setlength{\tabcolsep}{7.3pt}
\begin{tabular}[t]{@{}l|rrrr@{}}
    \toprule
    \textbf{Corpus}             & \textbf{Directions}   & \textbf{Train}    & \textbf{Dev} (dev2010)  & \textbf{Test} (tst2010) \\
    \midrule
    \multirow{20}{*}{IWSLT17}   & $\De \rightarrow \En$ & 206k              &    888        & 1568 \\
                                & $\De \rightarrow \It$ & 205k              &    923        & 1567 \\
                                & $\De \rightarrow \Nl$ & 0                 &    1001       & 1567 \\
                                & $\De \rightarrow \Ro$ & 201k              &    912        & 1677 \\
    \cmidrule[0.5pt]{2-5}
                                & $\En \rightarrow \De$ & 206k              &    888        & 1568 \\
                                & $\En \rightarrow \It$ & 231K              &    929        & 1566 \\
                                & $\En \rightarrow \Nl$ & 237k              &    1003       & 1777 \\
                                & $\En \rightarrow \Ro$ & 220k              &    914        & 1678 \\
    \cmidrule[0.5pt]{2-5}
                                & $\It \rightarrow \De$ & 205k              &    923        & 1567 \\
                                & $\It \rightarrow \En$ & 231k              &    929        & 1566 \\
                                & $\It \rightarrow \Nl$ & 205k              &    1001       & 1669 \\
                                & $\It \rightarrow \Ro$ & 0                 &    914        & 1643 \\
    \cmidrule[0.5pt]{2-5}
                                & $\Nl \rightarrow \De$ & 0                 &    1001       & 1779 \\
                                & $\Nl \rightarrow \En$ & 237k              &    1003       & 1777 \\
                                & $\Nl \rightarrow \It$ & 233k              &    1001       & 1669 \\
                                & $\Nl \rightarrow \Ro$ & 206k              &    913        & 1680 \\
    \cmidrule[0.5pt]{2-5}
                                & $\Ro \rightarrow \De$ & 201k              &    912        & 1677 \\
                                & $\Ro \rightarrow \En$ & 220k              &    914        & 1678 \\
                                & $\Ro \rightarrow \It$ & 0                 &    914        & 1643 \\
                                & $\Ro \rightarrow \Nl$ & 206k              &    913        & 1680 \\
    \cmidrule[0.5pt]{1-5}
    \multirow{20}{*}{IWSLT17$^\star$}
                                & $\De \rightarrow \En$ & 124k              &    888        & 1568 \\
                                & $\De \rightarrow \It$ & 0                 &    923        & 1567 \\
                                & $\De \rightarrow \Nl$ & 0                 &    1001       & 1567 \\
                                & $\De \rightarrow \Ro$ & 0                 &    912        & 1677 \\
    \cmidrule[0.5pt]{2-5}
                                & $\En \rightarrow \De$ & 124k              &    888        & 1568 \\
                                & $\En \rightarrow \It$ & 139k              &    929        & 1566 \\
                                & $\En \rightarrow \Nl$ & 155k              &    1003       & 1777 \\
                                & $\En \rightarrow \Ro$ & 128k              &    914        & 1678 \\
    \cmidrule[0.5pt]{2-5}
                                & $\It \rightarrow \De$ & 0                 &    923        & 1567 \\
                                & $\It \rightarrow \En$ & 139k              &    929        & 1566 \\
                                & $\It \rightarrow \Nl$ & 0                 &    1001       & 1669 \\
                                & $\It \rightarrow \Ro$ & 0                 &    914        & 1643 \\
    \cmidrule[0.5pt]{2-5}
                                & $\Nl \rightarrow \De$ & 0                 &    1001       & 1779 \\
                                & $\Nl \rightarrow \En$ & 155k              &    1003       & 1777 \\
                                & $\Nl \rightarrow \It$ & 0                 &    1001       & 1669 \\
                                & $\Nl \rightarrow \Ro$ & 0                 &    913        & 1680 \\
    \cmidrule[0.5pt]{2-5}
                                & $\Ro \rightarrow \De$ & 0                 &    912        & 1677 \\
                                & $\Ro \rightarrow \En$ & 128k              &    914        & 1678 \\
                                & $\Ro \rightarrow \It$ & 0                 &    914        & 1643 \\
                                & $\Ro \rightarrow \Nl$ & 0                 &    913        & 1680 \\
    \bottomrule
\end{tabular}
\caption{Data statistics for IWSLT17 and IWSLT17$^\star$.
Note that training data in IWSLT17$^\star$ was restricted to only $\En \leftrightarrow \{\De, \It, \Nl, \Ro\}$ directions and cleaned from complete pivots through \En, which also reduced the number of parallel sentences in each supervised direction.}
\label{tab:dataset-statistics}
\vspace{56ex}
\end{table}

\end{document}